\documentclass{article} % For LaTeX2e

\usepackage[numbers,square,sort&compress]{natbib}
\usepackage{algorithm} 
\usepackage{algpseudocode}
\usepackage{amsmath,amssymb,amsthm}
\usepackage{amsmath}
\usepackage{amssymb}
\usepackage{amssymb}
\usepackage{amsthm}
\usepackage{amsthm}
\usepackage{booktabs}
\usepackage{enumitem}
\usepackage{graphicx}
\usepackage{hyperref}
\usepackage{arxiv}
\usepackage{makecell}
\usepackage{mathrsfs}
\usepackage{pifont}
\usepackage{thmtools}
\usepackage{tipa}
\usepackage{tipa}
\usepackage{url}

\setlist{nosep}

\declaretheorem[name=Theorem]{Theorem}
\declaretheorem[name=Definition]{Definition}
\declaretheorem[name=Trajectory Policy Gradient Theorem]{MainTheorem}
\declaretheorem[name=Lemma]{Lemma}
\declaretheorem[name=Corollary]{Corollary}

\declaretheorem[name=Remark]{Remark}

\newcommand{\TrajectoryTheorem}{\textit{Trajectory Policy Gradient Theorem~}}
\newcommand{\ZeroAssumption}{\textit{Zero-Reward Assumption }}

\newcommand{\Exp}{\mathbb{E}}
\newcommand{\Var}{\mathbf{Var}}
\newcommand{\Cov}{\mathbf{Cov}}

\newcommand{\Path}{\mathbf{W}}
\newcommand{\Patht}{\mathbf{W}^{(t)}}
\newcommand{\Pathk}{\mathbf{W}^{(k)}}
\newcommand{\PathtPlusOne}{\mathbf{W}^{(t+1)}}

\newcommand{\PolicyProb}{\pi_{\theta}}
\newcommand{\PolicyProbt}{\pi(w_t|\Path_{0,t-1})}
\newcommand{\Prompt}{w_0}
\newcommand{\VFunc}{V_{\PolicyProb}}
\newcommand{\QFunc}{Q_{\PolicyProb}}
\newcommand{\GFunc}[1]{G(#1)}
\newcommand{\RM}{RM}
\newcommand{\Eqn}[1]{Eqn. (#1)}

\title{Response-Level Rewards Are All You Need for Online Reinforcement Learning in LLMs: A Mathematical Perspective}

\author{
Shenghua He\textsuperscript{\bf{*}}$^{1}$, Tian Xia\textsuperscript{\bf*}$^{2}$, Xuan Zhou$^{*1}$, Hui Wei$^{3}$\\
    $^1$ Amazon, $^2$ PAII Inc., $^3$ UC Merced\\
    \texttt{\{shenghh2015, TianXia0209, zhouxss093, wllnyubupt\}@gmail.com 
    }
}

\begin{document}
\maketitle
\footnotetext{Equal Contribution. Tian Xia originated the idea and provided the mathematical proof; Shenghua He extended the idea and derived key conclusions; Xuan Zhou provided invaluable feedback on the proof.}

\begin{abstract}
We study a common challenge in reinforcement learning for large language models (LLMs): \emph{the Zero-Reward Assumption}, where non-terminal actions (i.e., intermediate token generations) receive zero task-specific immediate reward, while only the final token receives a reward for the entire response. This assumption arises frequently in practice, as precise token-level rewards are often difficult or infeasible to obtain in LLM applications.
Recent work has addressed the limitations of this assumption in two major directions. One stream, including methods such as GRPO and ReMax, simplifies the critic network under the belief that zero intermediate rewards render it less influential. Another stream argues that the lack of fine-grained token-level feedback hinders effective credit assignment and attempts to model explicit token-level rewards.

In this work, we provide a unifying theoretical perspective. We introduce the \emph{Trajectory Policy Gradient Theorem}, which shows that the policy gradient based on true, unknown token-level rewards can be unbiasedly estimated using only a response-level reward model, regardless of whether the Zero-Reward Assumption holds or not, for algorithms in the REINFORCE and Actor-Critic families.
This result reveals that widely used methods such as PPO, GRPO, ReMax, and RLOO inherently possess the capacity to model token-level reward signals, offering a theoretical justification for response-level reward approaches. Our findings pave the way for more practical, efficient LLM fine-tuning, allowing developers to treat training algorithms as black boxes and focus on improving the response-level reward model with auxiliary sub-models.
We also offer a detailed analysis of popular RL and non-RL methods, comparing their theoretical foundations and practical advantages across common LLM tasks. Finally, we propose a new algorithm: \textbf{T}oken-\textbf{Re}inforced \textbf{P}olicy \textbf{O}ptimization (\textbf{TRePO}, /\textprimstress trep\textschwa\textupsilon /), a theoretically grounded method that is simpler than PPO, matches GRPO in memory efficiency, and holds promise for broad applicability.
\end{abstract}
\newpage
\section{Introduction}
\label{sec:intro}
With the rapid advancement of large language models (LLMs), reinforcement learning (RL) has become a key post-training technique for enhancing their general instruction-following abilities (e.g., GPT-4o~\cite{hurst2024gpt}, Gemini~\cite{team2023gemini}, Qwen2.5~\cite{yang2024qwen2}, Nova~\cite{intelligence2024amazon}) and advanced reasoning skills (e.g., O-1~\cite{jaech2024openai}, Deepseek-R1~\cite{guo2025deepseek}).

Proximal Policy Optimization (PPO)\cite{schulman2017proximal}, widely used in reinforcement learning from human feedback (RLHF), has been a foundational RL algorithm for post-training LLMs. 
It follows a general actor-critic architecture, where the policy model is the LLM being trained and the critic model estimates the expected long-term reward of intermediate states to reduce gradient variance during optimization.
PPO has demonstrated strong empirical success across various tasks, including summarization~\cite{stiennon2020learning}, instruction following~\cite{ouyang2022training}, and dialogue~\cite{bai2022training}.
Despite its practical effectiveness, PPO requires training both the policy and critic models, even though only the policy is typically used during inference.

In practical LLM application scenarios, PPO is typically implemented based on an assumption referred to in this work as the \ZeroAssumption: any state-action leading to an intermediate state receives zero immediate reward, while the state-action leading to the terminal state receives the full response-level reward for the entire trajectory.
This is primarily because, in practice, it is often difficult, costly, or even infeasible to precisely define an immediate reward for each state-action pair, whereas a response-level reward is much more readily obtainable.
To date, many popular prior works have recognized the limitations of the flawed \ZeroAssumption and attempted to address it by different ideas, which are roughly categorized into two groups.

In the first group,  Re-Max\cite{li2023remax} states that~``$r(s_t, a_t)$ is non-zero only when the whole trajectory ends at $t = T$. This means that RLHF tasks are close to single-stage optimization problems since the rewards of the intermediate stages are 0~". 
Group Relative Policy Optimization (GRPO) \cite{shao2024deepseekmath} states that-``While in the LLM context, usually only the last token is assigned a reward score by the reward model, which may complicate the training of a value function that is accurate at each token". 
Leave-One-Out Reinforce (RLOO) \cite{ahmadian2024back} states that~``modeling of partial sequences is unnecessary since rewards from humans are only attributed to full generations, with no true rewards for any intermediary token generate". 
These works focus on modeling response-level rewards and eliminate the critic network from PPO, which is typically as large as the policy model. 
Accordingly, methods falling under the \ZeroAssumption are generally categorized as \textit{response-level} algorithms. 
However, while these approaches are strong from an engineering standpoint, they often overlook the theoretical shortcomings of the \textit{Zero-Reward Assumption}. 
As a result, there is a general lack of theoretical analysis surrounding them.

In another group, some works \cite{zhong2024dpo,chan2024dense,cui2025process,li2024r3hf} attempted to address it by approximating instant rewards for intermediate states using a learned response-level reward model in specific application scenarios.
Chan et al.\cite{chan2024dense} proposed an PPO variant called ABC, which distributes the response-level reward across tokens using attention maps from the reward model itself.
Reinforced Token Optimization (RTO)~\cite{zhong2024dpo} employs a DPO model that is trained with response-level labels as an implicit token-level reward model to generate token-level rewards for PPO training.
More recently, Cui et al.~\cite{cui2025process} introduced a framework called Process Reinforcement through Implicit Rewards (PRIME), which learns an online implicit process reward model from outcome labels using a cross-entropy loss during the RL training for LLM mathematical tasks. 
These approaches typically rely on a second assumption, which we refer to as \textit{partial-reward assumption}: \textit{a token-level reward model can be learned from response-level preference or outcome labels and can provide rich and reliable feedback for partial trajectories}.
However, the partial-reward assumption is largely heuristic, even though it does provide token-level rewards in form.
In addition, RL approaches based on this assumption typically require training a response-level reward model to approximate token-level instant rewards, which greatly limits their practical generality.
Though, these works brings improvements over PPO, yet there are still no clear analysis of where the improvement comes from.
For example, is it from equipping the new methods with token-level capability, or bringing in additional information actually?

In this study, we revisit the fundamentals of RL to re-derive all formulas relevant to RL for LLMs and explore a general RL solution.
% with the aim of addressing the limitations of RL approaches under the \ZeroAssumption discussed above in a general and principled manner. 
Our contributions span three main aspects as described below:

\textbf{Trajectory Policy Gradient Theorem:} 
We propose a \textit{trajectory policy gradient theorem}, which demonstrates in LLMs application scenarios, the \ZeroAssumption used in REINFORCE and Actor-Critic RL algorithms, combined with a response-level reward model, can provide an unbiased estimate of any unknown token-level rewards.
This theorem offers both a theoretical guarantee and practical guidance for developing RL algorithms across a wide range of LLM applications—where token-level rewards are difficult, costly, or even impossible to define precisely, while response-level rewards are readily available. Such scenarios include question answering, mathematical reasoning, multi-agent training, and more.

\textbf{Theoretic analysis of RL methods:} 
Based on the proposed theorem, we theoretically analyze a set of popular RL and RL-free algorithms, e.g., PPO, GRPO, ReMax, RLOO and DPO. 
We conclude that, even under the \ZeroAssumption, PPO is able to explicitly models task-specific token-level information, and other algorithms do as well.
We further derive their core differences in terms of formulas, and conjecture PPO has a theoretical advantage over others, in terms of smaller approximation deviation from the exact RL objective and more accurate baseline selection for policy gradient variance reduction.
This finding might help explain why PPO has generally demonstrated advantages over DPO~\cite{xu2024dpo} and leads us to predict that PPO may also have the potential to outperform GRPO, ReMax, and RLOO in broader application scenarios.
Besides the generality from the \TrajectoryTheorem, it also provides the flexibility for practical applications so that more attention can be put on the design of the reward model, which simplifies the development and lowers the learning curve of rapidly applying RL in various applications.
For instance, in a mathematical reasoning task, various aspects and levels of granularity—such as the correctness of the final answer (binary outcomes), the correctness of intermediate steps (if detectable), and the formatting of the generated text—are uniformly aggregated into a response-level scalar reward, while token-level information is still effectively modeled during RL for LLM alignment.

\textbf{General token-level RL:} 
This new theorem removes both the \ZeroAssumption and \textit{partial-reward assumption}, and thus guides us to propose a theoretically grounded token-level RL approach, \textbf{T}oken-\textbf{Re}inforced \textbf{P}olicy \textbf{O}ptimization (TRePO, /\textprimstress trep\textschwa\textupsilon/), for LLM alignment, making it broadly applicable to any LLM tasks in which intermediate steps play a critical role but are difficult or infeasible to evaluate precisely.
TRePO is conceptually simpler than PPO, as it does not require an extra critic model to estimate the baseline, and is therefore as memory-efficient as GRPO. 
Although it introduces some engineering complexity due to additional sampling, we propose several approximation techniques to enable efficient implementation.

\section{A Revisit of RL Basics in LLMs}
\label{sec:rl_basics}

\subsection{Problem Formulation of RL for LLMs}
In a general setup of RL for LLM alignment, language generation is modeled as a sequential Markov Decision Process (MDP) $\{\mathcal{S},\mathcal{A},\mathcal{R}, \pi\}$, where policy $\pi_\theta(a_t|s_{t-1})$ represents a language model parameterized by $\theta$ that generates a response $\mathbf{y}=[y_0, y_1, ..., y_{T}]$ autoregressively given an input $\mathbf{x}=[x_0,x_1,...,x_{M-1}]$. 
The state at time step $t$, $s_t \in \mathcal{S}$, is the concatenation of the input prompt and the tokens generated up to that point:
$s_t = \textbf{x}; \, \textbf{y}_{<t} = \left[ x_0, \ldots, x_{M-1}, y_0, \ldots, y_{t-1} \right]$. 
At each time step, the action $a_t\in\mathcal{A}$ corresponds to generating the next token $y_t$ from a fixed vocabulary.

The MDP begins with the initial state $s_0=\textbf{x}$, and after each action, the environment transitions to the next state, $s_{t+1}=s_t:[a_t]$, by appending the action $a_t$ to the current state $s_{t-1}$.
During the token generation process, an instant reward $r_t\in\mathcal{R}$ is assigned to an intermediate $a_t$ starting from state $s_t$. 
A trajectory $\tau=(s_0,a_0,r_1,a_1,s_1,...)$ is therefore a sequence of state-action pairs, starting from the input prompt until the terminate action [EOS] token.

The cumulative future return starting from $s_t$ is then defined as $R(\tau)=\sum_{t=0}^{T-1}\gamma^{T-1}r_t$, where $\gamma$ is a discounted factor that penalizes the delayed rewards.
In many LLM applications, $\gamma$ is commonly set to 1. Differently, in our study, we do not need to specify $\gamma$, as our proposed theorem (see Section~\ref{sec:method}) is independent of its value.
Furthermore, in many LLM applications, such as mathematical reasoning, it is often difficult, costly, or even infeasible to precisely define an instant reward for each state-action pair.
Thus, existing RL approaches either enforce a \textit{zero-reward assumption} or a \textit{partial-reward assumption} (explained in Section~\ref{sec:intro}).
In contrast, our proposed theorem removes all such assumptions and requires only the cumulative return $R(\tau)$ for each trajectory $\tau$, while still being capable of modeling token-level information.

Given the definition of $R(\tau)$, the $V(s_t)$ is the value function of state $s_t$ defined as the expected cumulative future return starting from state $s_t$ and following the policy $\pi$: $V(s_t) = \mathbb{E}_{\tau \sim \pi_\theta} \left[ R(\tau) \mid s_0 = s_t \right]$.
$Q(s_t,a_t)$ is the state-action value of $(s_t, a_t)$ defined as expected cumulative future return starting from state $s_t$ after taking action $a_t$, accordingly $Q(s_t,a_t)=r_{t+1} + \gamma V(s_{t+1})$.

The goal of RL is the maximization of the expected accumulative future return of a MDP:
\begin{equation}
\max_{\theta} \mathcal{J}(\theta) = \max_{\theta} \mathbb{E}_t [\log\pi_\theta(a_t|s_{t}) Q(s_{t}, a_t)].
\label{eq:rl_objective}
\end{equation}

\subsection{Policy Gradient Methods}

Policy gradient methods for a general RL task can be used to optimize the objective in Equation~(\ref{eq:rl_objective}), typically involving two iterative steps:
\begin{itemize}[leftmargin=1.5em]
\item Gradient estimation: $\nabla \mathcal{J}(\theta_t) = \mathbb{E}_t [\nabla \log \pi_{\theta_t}(a_t |s_t) Q(s_t, a_t)]$
\item Gradient ascent: $\theta_{t+1} = \theta_t + \eta \nabla_{\theta_t} J(\theta_t)$, where $\eta$ is the learning rate.
\end{itemize}
In practice, a baseline $b(s_t)$ is strategically introduced to reduce the variance of the gradient estimate:
\begin{equation}
\nabla \mathcal{J}(\theta) = \mathbb{E}_t \left[ \nabla \log \pi_\theta(a_t | s_t) \left(Q(s_t, a_t) - b(s_t)\right) \right],
\end{equation}
which remains an unbiased estimator of $\nabla \mathcal{J}(\theta)$, as $b(s_t)$ depends only on the state $s_t$ and is independent of the action distribution $\pi(a_t | s_t)$. Accordingly, the general RL goal is equivalent to:
\begin{equation}
\max_{\theta} \mathcal{J}(\theta) = \max_{\theta} \mathbb{E}_t [\log\pi_\theta(a_t|s_{t}) \left(Q(s_t, a_t) - b(s_t)\right)].
\label{eq:rl_objective2}
\end{equation}
In a canonical actor-critic RL architecture, $b(s_t)$ is set to the value function $V(s_t)$, although this is not strictly required.
\section{Trajectory Policy Gradient Theorem and TRePO}
\label{sec:method}
In this section, we present the mathematical foundation and framework of our proposed Trajectory Policy Gradient Theorem, along with the implementation details of the TRePO algorithm. Table \ref{tab:symbol_table} summarizes the standard notation used in reinforcement learning for LLMs, which we adopt throughout this section.

% Preview source code for paragraph 38
\begin{table}[H]
\label{tab:symbol_table}
\begin{centering}
\begin{tabular}{cc}
\hline 
Symbol & Meaning\tabularnewline
\hline 
$\Path$ & A full trajectory sampled for current optimization.\tabularnewline
T & The number of output tokens in $\Path$.\tabularnewline
$\Path_{a,b}$ & A substring in $\Path$.\tabularnewline
$w_{t}$ & The $t$-th token in $\Path$.\tabularnewline
$\Prompt$ & User prompt tokens.\tabularnewline
$\Patht$ & A full trajectory that shares the same prefix of $\Path$, s.t., $\Patht_{0,t-1}=\Path_{0,t-1}$.\tabularnewline
$r(w_{t})$ & The unknown reward for $w_{t}$, depending on $\Path_{0,t-1}$.\tabularnewline
$\VFunc(\Path_{0,t})$ & The expected future return from state $\Path_{0,t}$.\tabularnewline
$\QFunc(\Path_{0,t-1},w_{t})$ & $r(w_{t})+\gamma\VFunc(\Path_{0,t})$.\tabularnewline
$c(\Path_{0,t})$ & Any constant that only depends on $\Path_{0,t}$.\tabularnewline
$G_{t}(\Path)$ & Discounted reward, $G_{t}(\Path)=\sum_{k=t}^{T}\gamma^{k-t}r(w_{t})$.\tabularnewline
$\PolicyProb(w_{t}|\Path_{0,t-1})$ & The probability of generating $w_{t}$, which is represented by a LLM.\tabularnewline
$\RM(\Path)$ & Reward score defined on a full trajectory.\tabularnewline
$\gamma$ & The reward discount factor.\tabularnewline
$\lambda$ & Used in GAE to control the degree of TD and Monte Carlo. \tabularnewline
\hline 
\end{tabular}
\par\end{centering}
\caption{Common symbols used in RL for LLMs, which we adopt throughout this section.}

\end{table}

\subsection{Main theory}

\begin{Lemma}\label{llm-gradienttheorem}
In LLMs application scenarios, the policy gradient theorem with a discounted reward factor $\gamma$ would be
\begin{align}
\nabla\mathcal{J}(\theta) =\Exp_{\Path}\sum_{t=1}^{T}\gamma^{t-1}\QFunc(\Path_{0,t-1},w_{t})\nabla\log\pi_{\theta}(w_{t}|\Path_{0,t-1})
\end{align}
where ${\Path=\Prompt w_{1}w_{2}\ldots w_{T}\sim\PolicyProb(\Path|\Prompt)}$.
\end{Lemma}

To adapt to NLP and LLM application scenarios, we use more familiar notations.
In LLMs, each output token depends on the user prompt and all its previous output tokens, and the state $S$ in RL literature actually corresponds to a string of prompt and output tokens.
We use a special $\Prompt$ to denote the user's prompt string, and redefine the state as $\Path_{0,t-1}=\Prompt+w_{1}w_{2}\ldots w_{t-1}=\Prompt w_{1}w_{2}\ldots w_{t-1}$, where the operator + concatenates two consecutive strings or states.
A sub-string of $\Path$ can be denoted as $\Path_{i,j}=w_{i}w_{i+1}\ldots w_{j}$.
% =w_{1}w_{2}\ldots w_{T}
The state-action in LLMs means taking a token $w$ to append to a state $\Path_{0,t-1}$ under a policy $\pi_{\theta}(w|\Path_{0,t-1})$, resulting in a new state $\Path_{0,t-1}+w$ with the state transition probability $p(\Path_{0,t}+w|\Path_{0,t-1},w)=1$ and $\pi_{\theta}(\Path_{0,t-1}+w|\Path_{0,t-1})=\pi_{\theta}(w|\Path_{0,t-1})$ trivially. 
Meanwhile, the state-action generates an instant reward $r(\Path_{0,t-1},w)$, which often cannot be calculated precisely in practice.
$\VFunc(\Path_{0,t-1})$ denote the expected discounted future cumulative reward on state $\Path_{0,t-1}$, and $\QFunc(\Path_{0,t-1},w)$ is the expected discounted future cumulative reward after appending $w$ to the state $\Path_{0,t-1}$, s.t. $\QFunc(\Path_{0,t-1},w)=r(\Path_{0,t-1},w)+\gamma\VFunc(\Path_{0,t-1}+w)$ and $\VFunc(\Path_{0,t-1})=\Exp_{w\sim\PolicyProb(w|\Path_{0,t-1})}\QFunc(\Path_{0,t-1},w)$, where $\gamma$ is a discounted factor.

\begin{proof}

\begin{align}
\nabla\mathcal{J}(\theta)= & \nabla\VFunc(\Prompt)\\
= & \nabla\sum_{w_{1}}\pi_{\theta}(w_{1}|\Prompt)\QFunc(\Prompt,w_{1})\\
= & \sum_{w_{1}}(\QFunc(\Prompt,w_{1})\nabla\pi_{\theta}(w_{1}|\Prompt)+\pi_{\theta}(w_{1}|\Prompt)\nabla\QFunc(\Prompt,w_{1}))\\
= & \sum_{w_{1}}\pi_{\theta}(w_{1}|\Prompt)\left(\frac{\QFunc(\Prompt,w_{1})\nabla\pi_{\theta}(w_{1}|\Prompt)}{\pi_{\theta}(w_{1}|\Prompt)}+\nabla\QFunc(\Prompt,w_{1})\right)\\
= & \Exp_{w_{1}\sim\PolicyProb(w|\Prompt)}(\QFunc(\Prompt,w_{1})\nabla\log\pi_{\theta}(w_{1}|\Prompt)+\nabla\QFunc(\Prompt,w_{1}))\\
= & \Exp_{w_{1}\sim\PolicyProb(w|\Prompt)}(\QFunc(\Prompt,w_{1})\nabla\log\pi_{\theta}(w_{1}|\Prompt)+\nabla(r(\Prompt,w_{1})+\gamma\VFunc(\Prompt w_{1}))\mathclose{)}\\
= & \Exp_{w_{1}\sim\PolicyProb(w|\Prompt)}(\QFunc(\Prompt,w_{1})\nabla\log\pi_{\theta}(w_{1}|\Prompt)+\gamma\nabla\VFunc(\Prompt w_{1}))\\
= & \Exp_{w_{1}\sim\PolicyProb(w|\Prompt)}\left(\QFunc(\Prompt,w_{1})\nabla\log\pi_{\theta}(w_{1}|\Prompt\right)\\
 & +\gamma\Exp_{w_{2}\sim\PolicyProb(w|\Prompt w_{1})}\QFunc(\Prompt w_{1},w_{2})\nabla\log\pi_{\theta}(w_{2}|\Prompt w_{1})\\
 & +\gamma^{2}\nabla\VFunc(\Prompt w_{1}w_{2}))\\
= & \Exp_{\Path=\Prompt w_{1}w_{2}\ldots w_{T}\sim\PolicyProb(\Path|\Prompt)}\sum_{t=1}^{T}\gamma^{t-1}\QFunc(\Path_{0,t-1},w_{t})\nabla\log\pi_{\theta}(w_{t}|\Path_{0,t-1})
\end{align}
\end{proof}

\begin{Remark}
    The classical Policy Gradient Theorem \cite{sutton2018reinforcement} involves a stationary state distribution $\mu$, from which it is difficult to sample in practice \cite{wang2025analysis}.
    In the LLM applications, due to the quite different definition of the state, $s$ in classic RL and $\Path$ in LLMs, Lemma \ref{llm-gradienttheorem} shows samples can be directly drawn from $\pi_{\theta}$.
    This provides a great convenience for analyzing the RL theory and algorithms in LLMs.
\end{Remark}

\begin{Remark}\label{remark-Gt}
People often use the currently sampled $\Path$ and $G_t(\Path)=\sum_{k=t}^T r(\Path, w_k)$ to help estimate $\QFunc(\Path_{0,t-1},w_{t})$.
Based on the \ZeroAssumption, $G_t(\Path)$ on all time steps $t$ equals $\RM(\Path)$, which results in the core formula used in ReMax, GRPO, and RLOO.
\end{Remark}

\begin{Definition}\label{definition}
A reinforcement learning algorithm is said to have the capability of token-level modeling, if its policy gradient is an unbiased estimator of Lemma \ref{llm-gradienttheorem}.
\end{Definition}

Trivially, if a token-level reward model is available, then $\QFunc(\Path_{0,t-1},w_{t})$ can be unbiasedly estimated through Monte Carlo sampling, e.g., calculate $G_t(\Path)$ in Remark \ref{remark-Gt}.

If a token-level reward model is unavailable, the \ZeroAssumption is widely adopted. 
Intuitively, it does not make sense.
For example, humans can readily understand the core idea of a response after reading the first eighty percents of the words.
The information in the whole response is not condensed to the last word exclusively.

\begin{Lemma}\label{baseline}
Adding a state-dependent constant $c(\Path_{0,t-1})$ to Lemma \ref{llm-gradienttheorem} does not affect the expectation of the policy gradient.
\begin{align*}
    \nabla\mathcal{J}(\theta) =\Exp_{\Path}\sum_{t=1}^{T} \left( c(\Path_{0,t-1}) + \gamma^{t-1}\QFunc(\Path_{0,t-1},w_{t}) \right)\nabla\log\pi_{\theta}(w_{t}|\Path_{0,t-1})
\end{align*}
where ${\Path=\Prompt w_{1}w_{2}\ldots w_{T}\sim\PolicyProb(\Path|\Prompt)}$.
\end{Lemma}

\begin{proof}
    
\begin{align}
0= & c(\Path_{0,t-1})\cdot0\\
= & c(\Path_{0,t-1})\cdot\nabla\sum_{w_{t}}\PolicyProb(w_{t}|\Path_{0,t-1}))\\
= & c(\Path_{0,t-1})\cdot\sum_{w_{t}}\nabla\PolicyProb(w_{t}|\Path_{0,t-1}))\\
= & \Exp_{w_{t}\sim\PolicyProb(w_{t}|\Path_{0,t-1})}c(\Path_{0,t-1})\cdot\nabla\log\PolicyProb(w_{t}|\Path_{0,t-1}))
\end{align}

Let $A_{t}=\gamma^{t-1}\Exp_{w_{t}\sim\PolicyProb(w_{t}|\Path_{0,t-1})}\QFunc(\Path_{0,t-1},w_{t})\nabla\log\pi_{\theta}(w_{t}|\Path_{0,t-1})$, and add the last equation to finish the proof. 

\begin{align}
    A_{t}= \Exp_{w_{t}\sim\PolicyProb(w_{t}|\Path_{0,t-1})}(c(\Path_{0,t-1}) + \gamma^{t-1}(\QFunc(\Path_{0,t-1},w_{t}) )\nabla\log\pi_{\theta}(w_{t}|\Path_{0,t-1})
\end{align}

\end{proof}

\begin{Theorem}\label{Theorem-base}
In LLM application scenarios, when the response-level reward is defined in such a form $\RM(\Path=w_{1}w_{2}\ldots w_{T})=\sum_{t=1}^{T}\gamma^{t-1}r(\Path_{0,t-1},w_{t})$ for real token rewards $r(\Path_{0,t-1},w_{t})$ and some discounted reward factor $\gamma$, then the policy gradient can be unbiasedly estimated from mere response-level rewards, regardless of the values of $r(\Path_{0,t-1},w_{t})$ and $\gamma$.
\begin{align}
    \nabla\mathcal{J}(\theta) = \Exp_{\Path} \sum_{t=1}^{T} \Exp_{\PathtPlusOne} \RM(\PathtPlusOne)\nabla\log\pi_{\theta}(w_{t}|\Path_{0,t-1})
\end{align}
where $\Path \sim \PolicyProb(\Path|\Prompt)$;
$\Patht\sim\PolicyProb(\Path|\Prompt)$ is a full trajectory, s.t. $\Patht_{0,t-1}=\Path_{0,t-1}$;
the $\RM(\Path)$ can be defined on task-specific outcome or reasoning steps, or both. 
\end{Theorem}

% \todo{Should we delete it?}
% Each token $w_{t}$ contributes an existent yet unknown instant reward $r_{t}$. 
% Take an intuitive example, after reading the first 95 tokens of a response with 5 tokens left, we often still get almost all the information. 
% However, by the \ZeroAssumption, the first 95 tokens contribute nothing. 

It is possible to train an extra network to model each $r(w_{t})$, by fitting their summation towards the target response-level score, which is often available. 
But the accurate modeling of $r(w_{t})$ brings extra complexity. 
Instead, we resort to mathematical derivation. 

\begin{proof}
Within a given $\Path$, we use $r(w_{t})$ for short. 
% For conciseness, we use $G_{t-1}$ to replace $G(\Path_{1,t-1})$ and $b_{t}$ to replace $b(\Path_{1,t})$. 
     % Since the trajectory $\Path$ is randomly sampled, then any of its consecutive sub-trajectory is also randomly sampled, then we can use the real cumulative rewards starting on a state $\Path_{1,t-1}$, namely $G(\Path_{1,t-1})$, to approximate its expected $\QFunc(\Path_{1\ldots t-1},w_{t})$. 
At each time step $t$, we set $c(\Path_{0,t-1}) = \sum_{k=1}^{t-1}\gamma^{k-1}r(w_{k})$, which exclusively depends on state $\Path_{0,t-1}$, and apply Lemma \ref{baseline} to obtain

\begin{align}
\nabla\mathcal{J}(\theta)= & \Exp_{\Path}\sum_{t=1}^{T}\left(\sum_{k=1}^{t-1}\gamma^{k-1}r(w_{k}) + \gamma^{t-1}\QFunc(\Path_{0,t-1},w_{t})\right)\nabla\log\pi_{\theta}(w_{t}|\Path_{0,t-1})\nonumber \\
= & \Exp_{\Path}\sum_{t=1}^{T}\left(\sum_{k=1}^{t}\gamma^{k-1}r(w_{k})+\gamma^{t}\VFunc(\Path_{0,t})\right) \nabla\log\pi_{\theta}(w_{t}|\Path_{0,t-1})\\
\end{align}

On each time step $t$, we sample a full trajectory $\Patht$ depending on state $\Path_{0,t-1}$ under the policy $\pi_{\theta}$ , s.t. $\Patht_{0,t-1}=\mathbf{W}_{0,t-1}$. 
We use $M$ to denote the length, then $\Patht=\Path_{0,t-1} + \Patht_{t,M}$.
The $M$ differs in different sampled $\Patht$.
We define $\GFunc{\Patht_{t,M}}$ as the real cumulative rewards of trajectory $\Patht$ from time step $t$ to $M$.
Then we have

\begin{align}
= \Exp_{\Path}\sum_{t=1}^{T}\left(\Exp_{\Path^{t+1}}(\sum_{k=1}^{t}\gamma^{k-1}r(w_{k})+\gamma^{t}\GFunc{\Patht_{t,M}})\right) \nabla\log\pi_{\theta}(w_{t}|\Path_{0,t-1})
\end{align}

Since $\RM(\Path) \stackrel{\text{def}}{=} \sum_{k=1}^{T}\gamma^{k-1}r(w_{k})= \sum_{k=1}^t \gamma^{k-1} r(w_k) + \gamma^t G(\Path_{t,M})$,  we have a more meaningful form as
\begin{equation}
\nabla\mathcal{J}(\theta) = \Exp_{\mathcal{\mathbf{W}}}\sum_{t=1}^{T}\Exp_{\PathtPlusOne}\RM(\PathtPlusOne) \nabla\log\pi_{\theta}(w_{t}|\Patht_{0,t-1})
\end{equation}

\end{proof}

By deriving the policy gradient theorem from the very beginning for LLMs, we can clearly show that the policy gradient does not rely on exact token-level reward computation or the choice of $\gamma$, as long as we can sample extra trajectories on each time step to compute its expectation.

\begin{Remark}
    Regarding the $\gamma$ setting, research involving PPO commonly sets the discount factor $\gamma$ near 1 which often lead to optimal performance. 
    This is because they followed the classic policy gradient theorem, typically derived with $\gamma=1$ so that it can lead to a concise final form.
    Otherwise, the derived policy gradient theorem would be much complicated for general RL.
    % and also results in the loss of $\gamma^T$ factor in Lemma \ref{llm-gradienttheorem}.
    However, RL in LLM, as shown in Lemma \ref{llm-gradienttheorem}, is much simpler for an accurate analysis. 
    If $\gamma \neq 1$, it would lead to an undesired result.
\end{Remark}

\begin{proof}
    We can set $\gamma=1$ in Lemma \ref{llm-gradienttheorem}, and still use $\gamma \neq 1$ to rederive Theorem \ref{Theorem-base}.
    We can naturally obtain 
    \begin{equation}
        \nabla\mathcal{J}(\theta) = \mathbb{E}_{\mathcal{\mathbf{W}}}\sum_{t=1}^{T}\frac{1}{\gamma^{t-1}} \Exp_{\PathtPlusOne}\RM(\PathtPlusOne) \nabla\log\pi_{\theta}(w_{t}|\Patht_{0,t-1})
    \end{equation}
% \textcolor{red}{
%     \begin{equation}
%         \nabla\mathcal{J}(\theta) = \mathbb{E}_{\Path_{0,t-1}}\frac{1}{\gamma^{t-1}} \Exp_{\Path^{(t + 1)}}\RM(\Path^{(t + 1)}) \nabla\log\pi_{\theta}(w_{t}|\Path^{(t)}_{0,t-1})
%     \end{equation}}
    Now the problem surfaces.
    When we want to emphasize the contribution from earlier states by setting $\gamma<1$, this new formula will surprisingly do the contrary thing by putting more weight on later states.
\end{proof}

\begin{Corollary}\label{coro-grpo}
    REINFORCE-based RL algorithms, such as GRPO, ReMax, RLOO, have the capability of token-level modeling.
\end{Corollary}

\begin{proof}
    Their core formulas are typically as following
    \begin{align}
        \nabla\mathcal{J}(\theta) & \approx \Exp_{\Path} \sum_{t=1}^{T} \RM(\Path)\nabla\log\pi_{\theta}(w_{t}|\Path_{0,t-1})
    \end{align} 
    Since the current randomly sampled $\Path$ is one of $\Patht$, then $\RM(\Path)$ can be used as an unbiased estimator of $\Exp_{\Patht}\RM(\Patht)$, though there exists an apparent margin of error.
    Plugging into Theorem \ref{Theorem-base} and following Definition \ref{definition} to finish the proof.
\end{proof}

This is sort of counter-intuitive, as the weight on each time step becomes identical.
But if we consider all trajectories $\Path$, then the extra term $c(\Path_{0,t-1})$ added into \RM($\Path$), induced in the proof of Theorem \ref{Theorem-base}, would disappear, and $\RM(\Path)$ would restore into $G_t(\Path)$ with true token-level rewards.

Interestingly, the proof has nothing to do with \ZeroAssumption, on which GROP, ReMax are based.

\begin{Theorem}\label{theorem-base-baseline}
    Theorem \ref{Theorem-base} with a near-optimal baseline for the gradient variance reduction is 
    \begin{align}
        \nabla\mathcal{J}(\theta) & = \Exp_{\Path} \sum_{t=1}^{T} \left(\Exp_{\PathtPlusOne} \RM(\PathtPlusOne) - \Exp_{\Patht} \RM(\Patht)\right) \nabla\log\pi_{\theta}(w_{t}|\Path_{0,t-1})
    \end{align}
\end{Theorem}

\begin{proof}
    We consider any time step $t$, by setting an optimal state-dependent constant $b_t$ to minimize the variance of $\left(\Exp_{\PathtPlusOne} \RM(\PathtPlusOne) - b_t \right) \nabla\log\pi_{\theta}(w_{t}|\Path_{0,t-1})$ in the distribution of  $\PolicyProbt$.
    
    Let a variate $X_1=\Exp_{\PathtPlusOne} \RM(\PathtPlusOne)$, and a second $X_2=\nabla\log\pi_{\theta}(w_{t}|\Path_{0,t-1})$.
    Our goal is to choose a constant b to minimize the variance of $X_1X_2 - b_t\cdot X_2$.
    
    As $\Exp_{\PolicyProbt} X_2 =0$, with the same trick in proving Lemma \ref{baseline}, we know this extra $b_t$ does not influence the expectation of $X_1X_2$, but only the variance.
    
    \begin{align*}
        \Var(X_1X_2 - bX_2)             = &\Var(X_1X_2) + \Var(bX_2) - 2\Cov(X_1X_2, bX_2)\\  
                                        = &\Var(X_1X_2) + b^2 \Var(X_2) - 2b\Cov(X_1X_2, X_2)\\
        \nabla_{b}\Var(X_1X_2-bX_2)     = & 2b\Var(X_2) - 2 \Cov(X_1X_2,X_2)=0\\
                                 b      = & \frac{\Cov(X_1X_2,X_2)}{\Var(X_2)}\\
    \end{align*}

    The optimal $b$ has a closed form and can be accurately estimated as long as there are enough samples from $X_1$ and $X_2$.
    Yet in practice, we can assume $X_1$ and $X_2$ are dependent so that it will lead to an elegant and easy-to-compute form.
    Then we have
    \begin{align*}
                                 b      = & \frac{\Exp X_1 \Cov(X_2,X_2)}{\Var(X_2)}\\
                                        = & \Exp X_1\\
                                        = & \Exp_{\PolicyProbt} \Exp_{\PathtPlusOne} \RM(\PathtPlusOne)\\
    \end{align*}
    As $\PathtPlusOne$ is defined as any full trajectory with its first $t$ tokens equal to that of the current sampled $\Path$, then $b$ equals $\Exp_{\Patht} \RM(\Patht)$.
     \begin{align*}
                                 b      = & \Exp_{\PolicyProbt} \RM(\Patht)\\
    \end{align*}
    
\end{proof}

\begin{Remark}
    The baseline $b$ in GRPO, ReMax and RLOO for the gradient reduction is estimated on a group of $\Path$, or one greedily sampled $\Path$.
    This same estimation is used for all time steps and states to remove the critic network in PPO for GPU memory saving.
    It is equivalent to our derived baseline term $\Exp_{\PolicyProbt} \RM(\Patht)$ with $t=1$ in Theorem \ref{theorem-base-baseline}. 
    Within reasoning-intensive applications, good states are considerably more likely to yield better trajectories than bad states, which corresponds to very different baselines. 
    Thus, the baseline setting following Theorem \ref{theorem-base-baseline} and PPO, is more likely to lead to a smoother policy gradient than that from GRPO, ReMax and RLOO.
\end{Remark}

    Besides, the commonly used advantage normalization is a common implementation trick in policy gradient methods, including PPO and REINFORCE.
    This setting does not change the expected policy gradient and is kind of equivalent to rescaling the learning rate by the empirical standard deviation of the advantages.
    But it does rescale the gradient in the current batch and can improve training stability and variance properties.

\begin{MainTheorem}\label{theorem-main}
    In LLM application scenarios, the policy gradient from true unknown token-level rewards can be unbiasedly estimated using a response-level reward model in REINFORCE and Actor-Critic RL algorithms, regardless of whether the \ZeroAssumption is applied.
\end{MainTheorem}

\begin{proof}
    Regarding REINFORCE RL algorithms, Theorem \ref{Theorem-base} and Corollary \ref{coro-grpo} prove they are not necessarily based on the \ZeroAssumption. 
    In practice, the Remark \ref{remark-Gt} looks more familiar and intuitive based on \ZeroAssumption, and it leads to the same conclusion, too.

    Regarding Actor-Critic RL algorithms, their core formulas with a sampled trajectory $\Path$ at a time step $t$ are typically 
    \begin{align}
        \left(\QFunc(\Path_{0,t-1},w_{t}) - \VFunc(\Path_{0,t-1})\right) \nabla\log\pi_{\theta}(w_{t}|\Path_{0,t-1}) \label{PPO-form}
    \end{align}

    The \ZeroAssumption assigns the last action with the response-level reward $\RM(\Path)$, which suggests $\gamma=1$; otherwise, it should be $\gamma^{T-1}\RM(\Path)$.
    
    As any $\PathtPlusOne$ in Theorem \ref{theorem-base-baseline} contributes to the $\QFunc(\Path_{0,t-1},w_{t})$ here, and any $\Path$ contributing to $\QFunc(\Path_{0,t-1},w_{t})$ accurately corresponds to a $\PathtPlusOne$ in Theorem \ref{theorem-base-baseline}.
    Thus, the calculation of $\RM(\Path)$ from true unknown token-level rewards has no difference from that on the \ZeroAssumption.
    Use the same analysis on the second term $\RM(\PathtPlusOne)$ and $\VFunc(\Path_{0,t-1}))$ to reach the same conclusion.
    
    As a result, Theorem \ref{theorem-base-baseline}, based on any real unknown token-level rewards and the discounted factor $\gamma$, is equivalent to \Eqn{\ref{PPO-form}} from Actor-Critic algorithms, based on the \ZeroAssumption and $\gamma=1$.
\end{proof}

This analysis tells us that the popular PPO, GROP, ReMax and RLOO actually have a token-level modeling capability, even if they are based on the \ZeroAssumption and response-level rewards.
This also suggests that direct improvement of an RL system can be conducted through enhancing its response-level reward model. 

\begin{Remark} \label{rm-usage}
    In the case where response-level reward models are jointly used with token-level reward models, we can merge them into a response-level reward model by calculating a single reward value for the entire trajectory $\textbf{W}$, following \Eqn{\ref{stronger-RM}}.
    
    \begin{align}
        \RM(\Path) = \sum_k \text{Response-level-RM}_k (\Path) + \sum_k \text{Token-level-RM}_k (\Path)    \label{stronger-RM}
    \end{align}
    where the weight of each reward model is absorbed into the model itself. 

    We can also treat them as token-level models by using the \ZeroAssumption to tackle the response-level reward models, and, together with other true token-level reward models, calculate the exact $G_t(\Path)$ at each step in Remark \ref{remark-Gt}.
    
    The second treatment may bring an advantage in the early stage of training, yet the two treatments are theoretically equivalent.

    For example, the 3H-style response-level rewards, namely helpfulness, honesty, and harmlessness, are hard to decompose into each token.
    A KL-distance style reward is naturally computable at each token.
    Another example of a token-level reward is an error detector that has an exact position.
\end{Remark}

\begin{Remark}
    The implementation of PPO often adopts GAE \cite{gae} to calculate the advantage function, which uses an extra $\lambda$ to provide a continuous, smooth trade-off between high-bias/low-variance (small $\lambda$) and low-bias/high-variance (large $\lambda$).
    
    When $\lambda=0$, then $\QFunc(\Path_{0,t-1}, w_t) = r(w_t) + \gamma \VFunc(\Path_{0,t})$, which leads to a smallest gradient variance. However, especially in the early stage, $\VFunc(\Path_{0,t})$ is inaccurate, which would induce a so-called self-referential bias.
    When $\lambda=1$, then $\QFunc(\Path_{0,t-1}, w_t) = \RM(\Path)$, becoming the first term in Corollary \ref{coro-grpo}, then the only difference of PPO from GRPO is their baseline selections.
    
    In practice, $\lambda$ is typically set as 0.9-0.95 to provide more information than $\RM(\Path)$ that is used in GROP, ReMax and RLOO.
\end{Remark}

\begin{Remark}
    TRePO, direct implementation of Theorem \ref{theorem-base-baseline} which is introduced in the next section, and PPO might have an advantage in performance over GRPO, ReMax, RLOO and DPO for general LLMs applications, if they are efficiently trained.
    For example, TRePO depends on moderately sufficient sampling operations, and PPO requires the critic network to have a better pretrained initialization for those reasoning-heavy applications where the response-level reward is defined as binary values.
\end{Remark}
    
    This is founded on two analyses.
    The first one, TRePO and PPO are trained towards the exact goal in Theorem \ref{Theorem-base}, as opposed to the others being trained towards an approximate goal in Corollary \ref{coro-grpo}, namely $\Exp_{\PathtPlusOne} \RM(\Patht)$ vs. $ \RM(\Path)$.
    Under the same optimizer steps, the former has a more accurate and efficient training. 
    DPO, as an ingenious algorithm, optimizes a ranking objective instead, yet its raw objective before the transformation is equivalent to Corollary \ref{coro-grpo} with an extra KL loss.

    The second one, as analyzed in the last remark, TRePO and PPO may have smoother policy gradients than others.

    However, in practice, there are other factors affecting the theoretical analysis. 
    PPO relies on a well-initialized critic network, and this demand is especially stronger in applications where only a binary-valued reward model is available, typically in mathematical reasoning.
    TRePO does not depend on an extra critic network, but it brings a cost of sufficient sampling and engineering effort.
    Only when these requirements are met can they demonstrate a practical advantage from a theoretical perspective.

\begin{Remark}\label{remark-improvement}
    Roughly, there are three potential directions to improve an LLM RL system.  
    
    The first one is the more accurate estimation of $\Exp_{\Patht} \RM(\Patht)$ in Theorem \ref{theorem-base-baseline}. 
    One interesting work \cite{kazemnejad2024vineppo} uses sampling to derive the same algorithm as our TRePO, but they are based on the \ZeroAssumption.
    
    The second one seeks to translate response-level rewards into precise token-level rewards, which may enhance training performance, as there can be a gap between practical implementation and the proposed theory.
    One typical work \cite{chan2024dense} uses the attention map from the reward model to reallocate the response-level reward to each token.
    
    The third one is inducing additional information to form a stronger reward model like \Eqn{\ref{stronger-RM}}, which is the most popular direction.
    One typical work \cite{zhong2024dpo} uses DPO to provide more information for each token. 
\end{Remark}

\subsection{TRePO implementation}
Based on Theorem \ref{theorem-base-baseline}, we propose a general token-level algorithm, TRePO.

A response-level reward model is not necessarily exclusively defined on the task-specific outcome, such as the final result in math reasoning, though you can surely do this.
It can be defined on all output tokens, as long as they matter to the task and you have efficient means to score them, such as rewarding desired reasoning steps or punishing wrong ones.
Naturally, a token-level reward also belongs to a response-level reward, such as the KL distance rewards from a reference SFT LLM in PPO, and DPO-enhanced KL rewards in RTO \cite{zhong2024dpo}.

Considering that doing multiple samplings each time step is far from efficient and applicable, we could select a subset of them, denoted as $D$.
We discuss how to generate the quality $\Patht$ on a given time step $t \in D$. 
To improve the efficiency, we could consider those samples that fall within high-probability regions by adjusting the inference temperature starting from zero.
Setting the temperature as zero corresponds to a greedy sampling strategy, which is adopted in ReMax.
This greedy strategy might not be inefficient if we consider the scenarios where an LLM-based policy model is strong enough and relatively stable in policy decision making, as well as the reward model measures the quality of the whole output and returns non-binary values.
Then, for other scenarios like mathematical reasoning, we need to gradually enlarge the temperature and sample more trajectories to better approximate $\Exp_{\Patht} \RM(\Patht)$.

Suppose on a time step $t \in D$, we sample a predefined number of samples.
Since any $\Pathk$ is also one $\Patht$ s.t. $t <= k$, we use this formula to calculate the expectation as 

\begin{align}
   \Exp_{\Patht}\RM(\Patht) = \text{Average}(\{\RM(\Path^k)\}) ~~~~ s.t. ~~~ t <= k  \label{Advange-computation}
\end{align}

Another important issue is that the trajectory sampling for a large LLM is still slow.
Besides the sampling strategies discussed above, additional engineering techniques become quite important for an efficient implementation of TRePO, such as PagedAttention, quantization and precision Tuning, to improve the inference latency and throughput significantly, supported in vLLM \cite{kwon2023efficient}.

\begin{algorithm}
\caption{TRePO}
\begin{algorithmic}[1] % [1] enables line numbering
\Statex % Vertical spacing

\Procedure{CalculateAdvantage}{$\mathbf{W}$, $M$, $|\textrm{D}|$}
  \State $\textrm{D}$ $\gets$ $|\textrm{D}|$ time steps $t$ from random samples or application-specific designation.
  \State $\textrm{E}$ $\gets$ for each $t \in \textrm{D}$, sample predefined number $M$ of full trajectories $\Patht$, by increasing inference temperature from zero.
  \State $\textrm{F1}$ $\gets$ for each $t \in \textrm{D}$, estimate by \Eqn{\ref{Advange-computation}}.
  \State $\textrm{F2}$ $\gets$ for each neighboring $t_1, t_2 \in \textrm{D}$, estimate by Theorem \ref{theorem-base-baseline}
  \State $\textrm{F3}$ $\gets$ for each $t \notin \textrm{D}$, estimate Theorem \ref{theorem-base-baseline} by performing linear interpolation using the information in F2 from two nearest time steps from $\textrm{D}$.
  \State \Return $F2$ + $F3$.
\EndProcedure

\Statex % Vertical spacing
\Procedure{Main}{$|\textrm{D}|$, batch\_num, optimization\_num}:
\For{$i = 1$ \textbf{to} batch\_num}
    \State Randomly sample a batch of trajectories $\mathbf{W}$ from the current policy $\theta_{\pi}$.
    \State Call CalculateAdvantage($\mathbf{W}$) for each trajectory.
        \For{$j = 1$ \textbf{to} optimization\_num}
          \State Calculate the average clipped surrogate gradient loss in \Eqn{30}.
          \State Optimizer.step()
        \EndFor
\EndFor
\EndProcedure

\end{algorithmic}
\end{algorithm}

\section{Survey of RL Methodologies for LLMs}
\label{sec:survey}

In this section, we provide a detailed survey of widely used RL and non-RL methods for LLM alignment.
Based on the assumptions they typically rely on, we group these methods into two categories:
(1)~\textbf{zero-reward approaches}, and (2)~\textbf{partial-reward approaches}.
These correspond to methods based on the \textit{zero-reward assumption} and those based on \textit{the partial-reward assumption}, respectively.

In LLM alignment applications, a KL penalty term is typically added to the optimization objective to prevent the policy from drifting too far from the initial policy during training, thereby avoiding reward over-optimization issues.
This penalty can be applied at the token level or the response level. However, since it is task-irrelevant, we exclude it from the objectives of the methods surveyed here.

\subsection{Zero-reward Approaches}
This group of approaches includes PPO and its RL alternatives, such as GRPO, RLOO, and ReMax, and RL-free alternative, DPO.

\textbf{Proximal Policy Optimization (PPO)~\cite{schulman2017proximal}}: PPO is a policy gradient method first introduced by Schulman et al.~\cite{schulman2017proximal} for general RL tasks. 
It follows a standard actor-critic RL architecture, in which $b(s_t)$ is set to $b(s_t)= V(s_t)$ and a low-variance RL objective is maximized:
\begin{equation}
    \mathcal{J}(\theta) = \mathbb{E}_t [\log\pi_\theta(a_t|s_{t}) A(s_t, a_t)],
    \label{eq:actor_critic_objective}
\end{equation}
where $A(s_t, a_t) = Q(a_t|s_{t}) - V(s_t)$ is the advantage function.

Different from the standard actor-critic RL, PPO proposes a clipped surrogate of the RL objective to improve the stability and sampling efficiency during RL training in practice, which is defined as: 
% by penalizing dramatic changes to the policy that are too far from the current one. 
% The clipped surrogate objective is defined as follows:
\begin{equation}
    \mathcal{J}^{\text{PPO}}(\theta)=\mathbb{E}_t \left[ \min \left( r_t(\theta) \hat{A}(s_t,a_t), \ \text{clip}(r_t(\theta), 1 - \epsilon, 1 + \epsilon) \hat{A}(s_t,a_t) \right) \right],
\end{equation}
where $\hat{A}(s_t,a_t)$ is an estimate of the advantage function $A(s_t,a_t)$ and is typically estimated using generalized advantage estimation (GAE), which involves a learned critic model to predict the value of $V(s_t)$.
Besides, $r_t(\theta)=\frac{\pi(a_t|s_{t-1})}{\pi_{old}(a_t|s_{t-1})}$ is the probability ratio between the current and old policy. The current policy $\pi(a_t|s_t)$ refers to the policy being trained the old policy $\pi_{old}(a_t|s_t)$ refers to the one that from its nearest training step that is used for sampling trajectories during RL training.
$\epsilon$ (e.g., 0.2) is a small parameter to control the clip range of the ratio change.
This objective thus prevents the being-trained policy from moving too far from the old one, improving reinforcement learning stability.

% Due to great stability and simplicity, it has been widely used in a reinforcement learning from human feedback (RLHF) framework for LLM post-training. A typical RLHF method contains four models: a policy model, a value model, a reference model, and a reward model.

\textbf{Group Relative Policy Optimization (GRPO)~\cite{shao2024deepseekmath}}: Excluding KL reward, GRPO aims to optimize the same objective as PPO. Instead of using a learned critic model to estimate token-level advantage functions, it uses a Monte Carlo sampling strategy to estimate a response-level advantage function.
Specifically, for each prompt, a group of responses are sampled using $\pi_{old}(a_t|s_t)$ and then a normalized reward is used to estimate the response-level advantage function of $k$-th response in the group, $\hat{A}_k(s_t, a_t)$, following zero-reward assumption:
\begin{equation}
    \hat{A}_k(s_t, a_t) = \hat{A}_k = \frac{r_k - \frac{1}{K} \sum_{o=1}^{K} r_o}{\text{STD}}.
\end{equation}
In such estimation, all tokens in a response have the same estimated advantage functions, which are only dependent on their response-level reward, the mean, and the standard deviation of the sampled response group. This yielded a simplified RL objective:
\begin{equation}
    \mathcal{J}^{\text{GRPO}}(\theta)=\mathbb{E}_{k,t} \left[ \min \left( r_t(\theta) \hat{A}_k, \ \text{clip}(r_t(\theta), 1 - \epsilon, 1 + \epsilon) \hat{A}_k \right) \right],
\end{equation}

\textbf{REINFORCE Leave-One-Out (RLOO)~\cite{ahmadian2024back}}: Similar to GRPO, RLOO aimed to use Monte Carlo sampling strategy to sample a group of responses for each prompt and then use the Leave-One-Out method to estimate the baseline $b_k(s_0)$ for $k$-th response in the group:\\
\begin{equation}
b_k(s_t) = \frac{1}{K-1} \sum_{k=1}^{K} r_k,
\end{equation}
However, RLOO does not use an important sampling clipping strategy. Under the zero-reward assumption, the objective of RLOO is:
\begin{equation}
\mathcal{J}^{\text{RLOO}}(\theta)= \mathbb{E}_t [\log\pi_\theta(a_t|s_{t}) (r_k - \frac{1}{K-1} \sum_{k=1}^{K} r_k)]
\end{equation}

\textbf{ReMax}~\cite{li2023remax}: Re-Max and RLOO are very similar, except that Re-Max uses a greedy sampling strategy to generate a single trajectory and then uses the corresponding reward score to estimate $b(s_t)$ for all tokens:
\begin{equation}
b(s_t) = r_{greedy},
\end{equation}
Accordingly, the objective of Re-Max is:
\begin{equation}
\mathcal{J}^{\text{ReMax}}(\theta)= \mathbb{E}_t [\log\pi_\theta(a_t|s_{t}) (r - r_{greedy})]
\end{equation}

\textbf{DPO}~\cite{rafailov2023direct}: Unlike RL approaches that optimize the expected cumulative future return, DPO directly optimizes the policy $\pi_\theta(a_t | s_t)$ using response-level preference data via maximum likelihood estimation, under the zero-reward assumption. The loss function is defined as follows: 
\begin{equation}
\mathcal{L}_{\text{DPO}}(\theta) = -\mathbb{E}_{(x, y_w, y_l) \sim \mathcal{D}} \left[ 
\log \sigma \left( 
\beta \log \frac{\pi_\theta(y_w | x)}{\pi_{\text{ref}}(y_w | x)} 
- 
\beta \log \frac{\pi_\theta(y_l | x)}{\pi_{\text{ref}}(y_l | x)}
\right) 
\right],
\end{equation}
where $y_w$ and $y_l$ are the human more-preferred and human less-preferred response respectively; $\beta$ is the parameter that controls the penalty of KL regularization; $\sigma$ is a sigmoid function; $\pi_{\text{ref}}$ is the policy represented by reference model.

\subsection{Partial-reward Approaches}
This group of approaches includes RTO~\cite{zhong2024dpo}, ABC~\cite{chan2024dense}, R3HF~\cite{li2024r3hf}, and PRIME~\cite{cui2025process}.
Instead of following the zero-reward assumption, these methods are developed under the partial-reward assumption.
They assume that a reward model trained on response-level preference labels can provide rich token-level feedback for RL training.
Since the primary focus of these methods is on assigning token-level instant rewards using a reward model learned from response-level labels, our discussion aligns with this focus and assumes PPO as the underlying RL algorithm for LLM alignment.

\textbf{RTO}~\cite{zhong2024dpo}: In RTO, the partial-reward assumption is specified as follows: token-level instant rewards can be derived from a DPO model trained on response-level preference labels and used for PPO training. The instant reward is defined as:
\begin{equation}
    r(s_t, a_t) = \frac{\pi_{\text{DPO}(a_t|s_t)}}{\pi_{\text{ref}(a_t|s_t)}},
\end{equation}
where $\pi_{\text{DPO}(a_t|s_t)}$ and $\pi_{\text{ref}(a_t|s_t)}$ are the policy model and reference model used for DPO training, as described above.

\textbf{ABC}~\cite{chan2024dense}: In ABC, the partial-reward assumption is specified as follows: a learned response-level reward model possesses an internal capability to weight the importance of each token in a response based on the attention scores. The instant reward can be defined as:
\begin{equation}
r(s_t, a_t) = w_t~\text{RM}(\mathbf{x},\mathbf{y}),
\end{equation}
where $\mathbf{x}$ and $\mathbf{y}$ are the prompt and the generated response respectively.

\textbf{R3HF}~\cite{chan2024dense}: In R3HF, the partial-reward assumption is specified as follows: a learned response-level reward model can implicitly examine the quality of a partial trajectory and the instant reward can then be derived from this ability:
\begin{equation}
r(s_t, a_t) = \text{RM}(s_{t+1}) - \text{RM}(s_{t}),
\end{equation}
where $s_t = \mathbf{x}:y_0,y_1,...,y_t$ represents the partial trajectory ending with $y_t$.

\textbf{PRIME}~\cite{cui2025process}: In PRIME, the partial-reward assumption is specified as follows: a token-level reward can be obtained generally from a reward model learned from response-level labels. In their study, a token-level reward is defined similarly to that is RTO, but it is trained using binary cross-entropy loss in an online manner to address the potential domain shift issue commonly observed with a reward model trained with offline data. The instant reward model is defined as:
\begin{equation}
    r(s_t, a_t) = r_\phi(s_{t+1}) = \frac{\pi_{\phi}(a_t|s_t)}{\pi_{\text{ref}}(a_t|s_t)},
\end{equation}
where both the $\pi_{\phi}(a_t|s_t)$ and $\pi_{\text{ref}}(a_t|s_t)$ are initialized with the initial policy in the RL training. $\pi_{\text{ref}}(a_t|s_t)$ 
 is frozen while $\pi_{\phi}(a_t|s_t)$ is trained by minimizing a binary cross entropy loss:
\begin{equation}
\mathcal{L}_{\text{CE}}(\phi) = -\mathbb{E}_{(\mathbf{x}, \mathbf{y})} \left[ 
l(\mathbf{x,y}) \cdot \log \sigma(r_{\phi}(\mathbf{x},\mathbf{y}))+(1 - l(\mathbf{x,y}))) \cdot \log (1 - \sigma(r_{\phi}(\mathbf{x}, \mathbf{y}))) 
\right],
\end{equation}
where $l(\mathbf{x,y})$ is a binary label generated by a reliable outcome verifier or human annotator. As a result, the online reward model training strategy is practically limited to scenarios where labeling is efficient, such as mathematical reasoning.
\section{Potential Impacts and Future Work}
In our study, we propose a trajectory policy gradient theorem that establishes how the policy gradient with unknown token-level information can be unbiasedly estimated from response-level rewards in online RL for LLMs. 
Based on this foundation, we theoretically analyze a range of RL algorithms, all of which are commonly based the \ZeroAssumption for LLM alignment, and RL-free algorithm.
Furthermore, we introduce TRePO, a theoretically grounded and general token-level RL approach for LLM alignment.
The broader impacts and limitations of our study are discussed below.

\subsection{Potential Impacts}
\textbf{Theoretic analysis of RL methods}:
Based on our theoretical analysis, the popular RL algorithms under the \ZeroAssumption, such as PPO, GRPO, Re-Max and RLOO, have the capability of explicitly modeling task-specific token-level information.
As the raw objective of DPO is equivalent to that of GROP, Re-Max and RLOO, it has too.
We also reveal the theoretical advances of PPO and TRePO over others in terms of optimized objectives and policy gradient reduction techniques.
It may also provide a theoretically grounded explanation for observations reported in prior studies, as illustrated in the example below.

Here, we revisit a prior comprehensive study by Xu et al.~\cite{xu2024dpo}, which compares PPO and DPO, and aim to provide a further explanation for the findings reported in their work. In that study, the authors observed that PPO consistently outperforms DPO across a variety of tasks, including HH-RLHF and PK-Safety. They attribute this to the fact that DPO is trained using offline preference data, making it more sensitive to out-of-distribution samples and resulting in poorer generalization. This conclusion is further supported by their observation that iterative DPO, which incorporates updated training data, significantly improves performance and narrows the gap with PPO.

This explanation aligns well with the results on the HH-RLHF dataset, as shown in Figure~\ref{fig:hh-rlhf_ppo_dpo}. However, on more challenging datasets such as CodeContents, iterative DPO shows only limited improvement over standard DPO, and both methods fall significantly short of PPO’s performance, as shown in Figure~\ref{fig:codeContents_ppo_dpo}. This suggests that the offline training setup is not the only factor contributing to DPO’s inferior performance compared to PPO. Based on our theoretical analysis of PPO and DPO, PPO leverages task-specific, token-level information more efficiently. Therefore, for tasks in which intermediate steps play a critical role, such as in CodeContents, PPO substantially outperforms even iterative DPO, despite the latter mitigating some of DPO’s limitations.

\begin{figure}[http]
    \centering
    \includegraphics[width=0.9\linewidth]{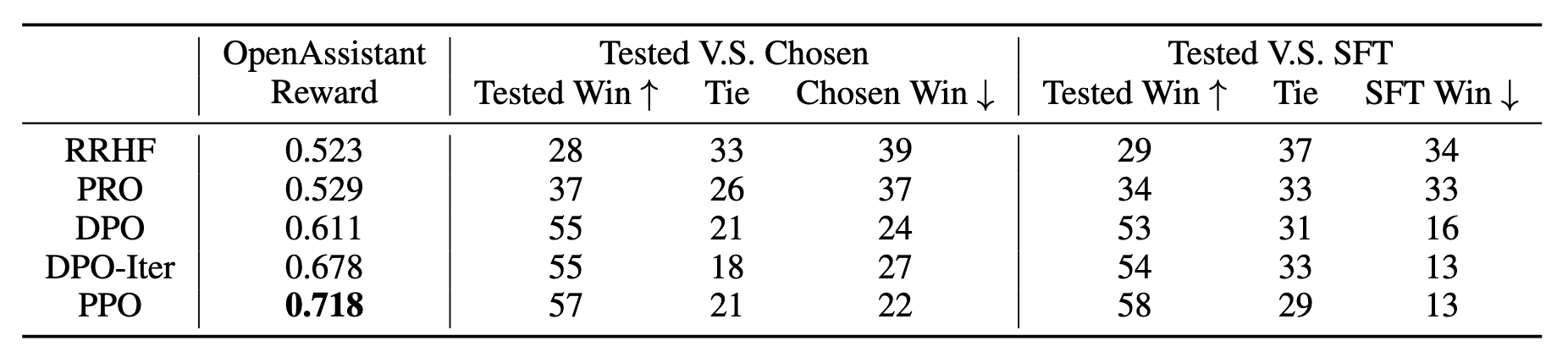}
    \caption{\small{Results on the HH-RLHF test set~\cite{xu2024dpo}. The evaluation metrics include the OpenAssistant rewards and the win rate of models against the chosen responses and SFT model outputs. The OpenAssistant reward model is not used during the training process. Note that DPO is trained on the preference data in the dataset, while Iter DPO is trained on self-generated responses, using a reward model for labeling.}}
    \label{fig:hh-rlhf_ppo_dpo}
\end{figure}

\begin{figure}[http]
    \centering
    \includegraphics[width=0.5\linewidth]{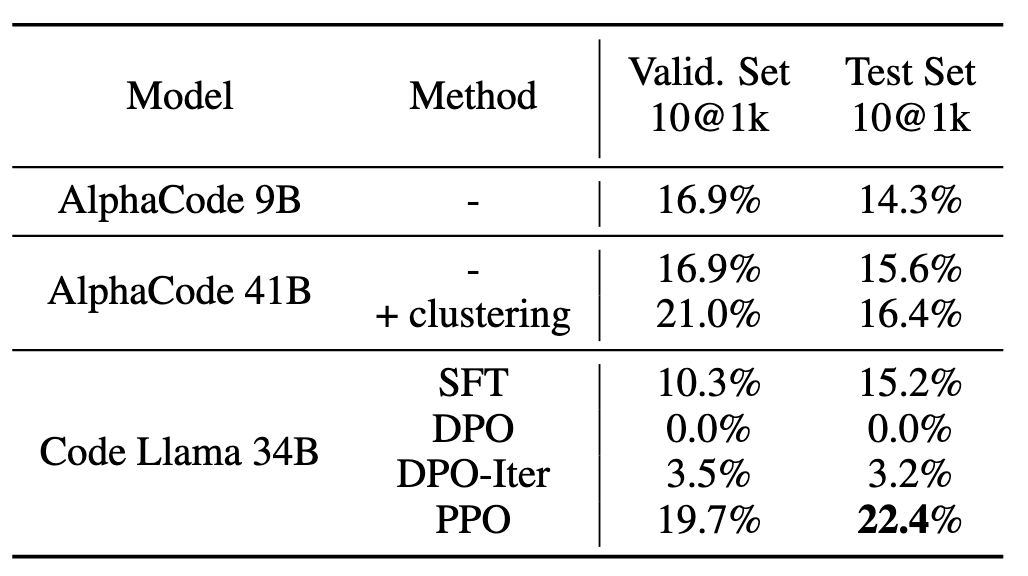}
    \caption{\small{Pass rate on CodeContests dataset~\cite{xu2024dpo}. “10@1k” means that 1000 samples will be evaluated on public tests in the problem description, and only 10 of them will be submitted for hidden tests. Only Python is used here for solving problems, while AlphaCode used both Python and C++.}}
    \label{fig:codeContents_ppo_dpo}
\end{figure}

\textbf{Potential applications of TRePO}: As the derivation of TRePO is closely based on the classic RL concepts, TRePO is widely applicable to those LLM-based application scenarios where classic online RL applies. 
Besides the common single-turn question-answering, we discuss the RL training in several multi-turn applications.

\paragraph{\textit{Multi-turn RL}}
The multi-turn RL problem for LLMs involves training an LLM agent to interact with an external environment (e.g., a human user) across multiple rounds \cite{zhou2025sweet, shani2024multi, abdulhai2023lmrl, zhou2024archer}. In this setting, the environment may contain hidden information not accessible to the agent in the initial task description but available through interaction \cite{zhou2025sweet}. The agent must generate a sequence of actions to achieve a specific goal. For example, when tasked with booking a flight to San Francisco, the LLM agent can ask the user about their preferences (e.g., price, arrival time) and complete the task through multiple steps, such as searching for flights and processing payment. This problem commonly arises in multi-step reasoning \cite{wang2024offline}, multi-step planning \cite{wei2025plangenllms}, multi-turn games \cite{abdulhai2023lmrl}, and multi-turn conversational tasks \cite{bai2024mt}. It requires more complex reasoning and planning capabilities than tasks focused solely on sentence generation or aligning with human preferences.

A key challenge in multi-turn RL problems is that rewards are typically provided only at the terminal state (i.e., after multiple interaction turns), making it difficult for single-turn RL methods (e.g., DPO \cite{rafailov2023direct}) to perform well. These methods lack explicit turn-wise credit assignment, which is essential for providing fine-grained supervision, identifying effective intermediate actions, and reinforcing those that contribute to achieving the final goal. This limitation leads to high variance as the number of turns increases, making convergence more difficult \cite{zhou2024archer, zhou2025sweet}. In contrast, our proposed TRePO algorithm enables step-wise reward estimation at any position in the sequence, based on the terminal reward. This capability makes TRePO well-suited for multi-turn RL settings involving LLM agents.

\paragraph{\textit{Multi-turn Medical Consultations}}
This is a popular application belonging to Multi-turn RL, with the difference that it is an off-line RL problem.
In this scenario, a doctor and a patient would converse about the disease symptoms for typically five to ten rounds.
The patient may upload photos as required, and most of the time, texts are used are describe the symptoms. 
The goal of this task is to train an AI model to simulate a human doctor to ask the correct as few questions as possible, and collect as much symptom information as possible.
Finally, a second human doctor or AI model can confidently make a diagnosis, based on the collected information. 

To convert this off-line RL problem into an online one, an LLM-based mock patient, who can answer new questions based on a specific patient's true information, is required. The work in \cite{ai-patient} states an LLM-based method that first extracts the patient's key information and then uses it as external knowledge to prompt the LLM to answer.

Designing a response-level reward model to measure the quality of a mock consultation can be considered from several aspects. The core one is the generation probability of the target disease for this patient, and other aspects include penalties for excess questions, duplicate questions, irrelevant questions, etc.

\paragraph{\textit{Multi-step Mathematical Reasoning and Coding}}
The mathematical reasoning resembles the coding in terms of the result measuring, which is generally treated as 0-1 loss, namely correct or incorrect. 
The significant difference between them is that, the reasoning steps in mathematical reasoning, from an LLM model output, might be wrong, yet they still may lead to a correct final answer 
\cite{lightman2023letsverifystepstep}.
In comparison, every output token from an LLM coding model matters to the final result.

We first consider this popular case with a 0-1 loss reward model. 
PPO may be underperforming, as PPO usually requires a well-trained reward model to initialize its critical network to stabilize the PPO training.
Some practical results from the LLM developers show that PPO with critic pretraining outperforms both GRPO and REINFORCE \cite{jian2024}.
TRePO does not rely on a critical network, and it might have an advantage in training stability and performance, at the cost of additional sampling and training time.

Empirically, process-based supervision leads to better performance than mere outcome-based supervision. 
We then consider the cases where additional process-based signals are available.
For example, in DeepSeek R1 \cite{guo2025deepseek}, the format of a reasoning is considered; 
in \cite{lightman2023letsverifystepstep,uesato2022solving}, the potential errors in some reasoning steps are considered.
Furthermore, any information that is considered to be a reward or a penalty, as long as they can be detected efficiently, can be simply added into the response-level reward model, and TRePO just serves as a black-box RL algorithm without requiring to designate the token positions of those fired signals.

% \textcolor{red}{
% 1. Discuss the broader impacts of our contributions:\\
%    (1) The trajectory policy gradient theorem provides practical guidance for RL research and applications:\\
%        ~~~~~a. It facilitates the labeling process in human-in-the-loop online RL for real-world LLM applications (e.g., AI doctor)\\
%        ~~~~~b. It offers a framework for integrating arbitrary feedback as universal binary signals into the RL training process.\\
%    (2) Our proposed TRePO algorithm can be generalized to a wide range of tasks.\\
%    (3) While prior studies focused on empirical comparisons between PPO and other alternatives, our work provides a theoretically grounded explanation for these observations.\\
% }

\subsection{Limitations and future work}
The focus of our study is primarily on the proposal of the trajectory policy gradient theorem and the theoretical analysis of RL methods.
Though we introduced a theoretically grounded RL approach TRePO for LLM alignment, which is conceptually simpler than PPO,  yet it also introduces a cost of additional sampling. 

In future work, we will evaluate and compare our proposed TRePO with other RL and RL-free approaches discussed in this study, in a variety of practical LLM tasks, such as mathematical reasoning, multi-turn medical dialogues, and multi-agent RL.

% \textcolor{red}{
% Discuss the limitations of our study and future directions}
% Discuss our prediction based on the results in previous studies. 
% \begin{itemize}
%     \item If we don't consider KL divergence, clipping strategy, and learned critic model, we have Re-Max=GRPO
%     \item PPO vs. DPO theoretic and experimental comparison. We would like to know why DPO works worse than PPO.
% \end{itemize}
\section{Related Work}

% \todo{Need to analyze the real cause leading to an improvement of DPO and other works.}
% \textcolor{red}{Mention some other studies that might be relevant but not directly related.}
Several related prior studies, classified into two categories below, can be further analyzed with the findings in our work.
The works in the first category actually combined the second and third directions from Remark \ref{remark-improvement}.
Are there improvements mainly from exact token position information or a stronger response-level reward?
Theories arising from our work tend to favor the latter; however, further experiments are needed to assess the practical contribution from the former.

\textbf{Explicit token-level rewards}: 
Several recent studies~\cite{wu2023fine,cao2024beyond,yoon2024tlcr,wang2023math,lyu2025exploring} have attempted to learn token‑level reward models explicitly from dense labels for RL‑based LLM alignment. 

For instance, Wu et al.~\cite{wu2023fine} explored the use of detailed human feedback at each intermediate step of an LLM’s generation, where such annotations were then used to train a fine-grained reward model for PPO training. However, this method incurs high costs for fine-grained human annotations, which limits its practical applications.

To address this limitation, some studies have investigated using LLMs themselves to provide feedback at intermediate steps or at the token level. For example, Cao et al.\cite{cao2024beyond} utilized an external LLM to generate sentence-level rewards via strategic prompting, thereby enhancing RL training for LLM alignment. In contrast, Yoon et al.\cite{yoon2024tlcr} leveraged mature LLMs to generate token-level preference labels and trained an explicit token-level reward model for RL, a direction known as token-level RL from AI feedback (RLAIF).
Despite the promising results, the reliability of these methods heavily depends on the quality of feedback provided by the LLMs. However, numerous recent studies~\cite{chiang2024chatbot,wei2024systematic} have shown that even the most advanced LLMs are susceptible to various vulnerabilities, such as position bias, length bias, and sensitivity to prompting templates. These limitations pose significant challenges to the practical adoption of token-level rewards derived from AI feedback.

From a different angle, several studies have investigated learning explicit token‑level reward models using automatically generated intermediate‑step annotations for tasks in which step‑wise labeling is feasible, such as mathematical reasoning. 
For instance, Wang et.~al~\cite{wang2023math} introduced an automatic annotation method that evaluates the quality of an intermediate reasoning step by sampling a batch of trajectories from that point and computing the proportion that ends in correct solutions—a quantity that can be viewed conceptually as a value function for the step. A process reward model is then trained from the step-level annotations for PPO-based LLM alignment for the mathematical reasoning task.
Lyu et.~al~\cite{lyu2025exploring} employed the same method to obtain a process reward to develop their RL method for mathematical reasoning tasks.

\textbf{Value functions}:
Another line of work~\cite{kazemnejad2024vineppo,chai2024ma} focuses on value functions of intermediate‑step states and may be relevant to our study.

Kazemnejad et al.\cite{kazemnejad2024vineppo} proposed a reinforcement learning approach called VinePPO for mathematical reasoning, which simplifies PPO by replacing the learned critic used in GAE computation with estimates obtained via Monte Carlo rollouts.
Although it demonstrates improved performance compared to PPO, its value function estimation is conceptually based on a zero-reward assumption, which significantly limits its generality and theoretical analysis.
In contrast, Chai et al.~\cite{chai2024ma} introduced MA-RLHF, which merges tokens into macro actions to reduce decision resolution and promote step-wise value function estimation in PPO, thereby alleviating issues caused by long temporal distances.

% \begin{itemize}
%     \item Token-level reward model learned from dense labels for LLMs: Fine-grained RLHF~\cite{kazemnejad2024vineppo}, ~\cite{lyu2025exploring}, ~\cite{wang2023math}
%     \item PPO-based works with spare rewards: VinePPO~\cite{kazemnejad2024vineppo}, MA-RLHF~\cite{chai2024ma}
%     % \item Other related works but not directly related 
%     % \item Token-level RL (token-level labeling):
%     % \item Step-wise Ranking
% \end{itemize}

\section{Acknowledgments}

We gratefully acknowledge the fundamental contributions of all the authors. 
Their ongoing, in-depth discussions have greatly facilitated the development of the proofs and conclusions.
Besides, we also acknowledge the invaluable feedback from Xianfeng Rui, Zhijing Ye, MingQi Li, Zijia Chen, Bo Peng, Zhining Gu, Jingyang Lin and Hengfeng Zhuang.

\bibliography{references}

\begin{thebibliography}{10}

\bibitem{hurst2024gpt}
Aaron Hurst, Adam Lerer, Adam~P Goucher, Adam Perelman, Aditya Ramesh, Aidan
  Clark, AJ~Ostrow, Akila Welihinda, Alan Hayes, Alec Radford, et~al.
\newblock Gpt-4o system card.
\newblock {\em arXiv preprint arXiv:2410.21276}, 2024.

\bibitem{team2023gemini}
Gemini Team, Rohan Anil, Sebastian Borgeaud, Jean-Baptiste Alayrac, Jiahui Yu,
  Radu Soricut, Johan Schalkwyk, Andrew~M Dai, Anja Hauth, Katie Millican,
  et~al.
\newblock Gemini: a family of highly capable multimodal models.
\newblock {\em arXiv preprint arXiv:2312.11805}, 2023.

\bibitem{yang2024qwen2}
An~Yang, Baosong Yang, Beichen Zhang, Binyuan Hui, Bo~Zheng, Bowen Yu,
  Chengyuan Li, Dayiheng Liu, Fei Huang, Haoran Wei, et~al.
\newblock Qwen2. 5 technical report.
\newblock {\em arXiv preprint arXiv:2412.15115}, 2024.

\bibitem{intelligence2024amazon}
Amazon Artificial~General Intelligence.
\newblock The amazon nova family of models: Technical report and model card.
\newblock 2024.

\bibitem{jaech2024openai}
Aaron Jaech, Adam Kalai, Adam Lerer, Adam Richardson, Ahmed El-Kishky, Aiden
  Low, Alec Helyar, Aleksander Madry, Alex Beutel, Alex Carney, et~al.
\newblock Openai o1 system card.
\newblock {\em arXiv preprint arXiv:2412.16720}, 2024.

\bibitem{guo2025deepseek}
Daya Guo, Dejian Yang, Haowei Zhang, Junxiao Song, Ruoyu Zhang, Runxin Xu,
  Qihao Zhu, Shirong Ma, Peiyi Wang, Xiao Bi, et~al.
\newblock Deepseek-r1: Incentivizing reasoning capability in llms via
  reinforcement learning.
\newblock {\em arXiv preprint arXiv:2501.12948}, 2025.

\bibitem{schulman2017proximal}
John Schulman, Filip Wolski, Prafulla Dhariwal, Alec Radford, and Oleg Klimov.
\newblock Proximal policy optimization algorithms.
\newblock {\em arXiv preprint arXiv:1707.06347}, 2017.

\bibitem{stiennon2020learning}
Nisan Stiennon, Long Ouyang, Jeffrey Wu, Daniel Ziegler, Ryan Lowe, Chelsea
  Voss, Alec Radford, Dario Amodei, and Paul~F Christiano.
\newblock Learning to summarize with human feedback.
\newblock {\em Advances in neural information processing systems},
  33:3008--3021, 2020.

\bibitem{ouyang2022training}
Long Ouyang, Jeffrey Wu, Xu~Jiang, Diogo Almeida, Carroll Wainwright, Pamela
  Mishkin, Chong Zhang, Sandhini Agarwal, Katarina Slama, Alex Ray, et~al.
\newblock Training language models to follow instructions with human feedback.
\newblock {\em Advances in neural information processing systems},
  35:27730--27744, 2022.

\bibitem{bai2022training}
Yuntao Bai, Andy Jones, Kamal Ndousse, Amanda Askell, Anna Chen, Nova DasSarma,
  Dawn Drain, Stanislav Fort, Deep Ganguli, Tom Henighan, et~al.
\newblock Training a helpful and harmless assistant with reinforcement learning
  from human feedback.
\newblock {\em arXiv preprint arXiv:2204.05862}, 2022.

\bibitem{li2023remax}
Ziniu Li, Tian Xu, Yushun Zhang, Zhihang Lin, Yang Yu, Ruoyu Sun, and Zhi-Quan
  Luo.
\newblock Remax: A simple, effective, and efficient reinforcement learning
  method for aligning large language models.
\newblock {\em arXiv preprint arXiv:2310.10505}, 2023.

\bibitem{shao2024deepseekmath}
Zhihong Shao, Peiyi Wang, Qihao Zhu, Runxin Xu, Junxiao Song, Xiao Bi, Haowei
  Zhang, Mingchuan Zhang, YK~Li, Y~Wu, et~al.
\newblock Deepseekmath: Pushing the limits of mathematical reasoning in open
  language models.
\newblock {\em arXiv preprint arXiv:2402.03300}, 2024.

\bibitem{ahmadian2024back}
Arash Ahmadian, Chris Cremer, Matthias Gall{\'e}, Marzieh Fadaee, Julia
  Kreutzer, Olivier Pietquin, Ahmet {\"U}st{\"u}n, and Sara Hooker.
\newblock Back to basics: Revisiting reinforce style optimization for learning
  from human feedback in llms.
\newblock {\em arXiv preprint arXiv:2402.14740}, 2024.

\bibitem{zhong2024dpo}
Han Zhong, Guhao Feng, Wei Xiong, Xinle Cheng, Li~Zhao, Di~He, Jiang Bian, and
  Liwei Wang.
\newblock Dpo meets ppo: Reinforced token optimization for rlhf.
\newblock {\em arXiv preprint arXiv:2404.18922}, 2024.

\bibitem{chan2024dense}
Alex~J Chan, Hao Sun, Samuel Holt, and Mihaela Van Der~Schaar.
\newblock Dense reward for free in reinforcement learning from human feedback.
\newblock {\em arXiv preprint arXiv:2402.00782}, 2024.

\bibitem{cui2025process}
Ganqu Cui, Lifan Yuan, Zefan Wang, Hanbin Wang, Wendi Li, Bingxiang He, Yuchen
  Fan, Tianyu Yu, Qixin Xu, Weize Chen, et~al.
\newblock Process reinforcement through implicit rewards.
\newblock {\em arXiv preprint arXiv:2502.01456}, 2025.

\bibitem{li2024r3hf}
Jiahui Li, Tai-wei Chang, Fengda Zhang, Kun Kuang, and Long Chen.
\newblock R3hf: Reward redistribution for enhancing reinforcement learning from
  human feedback.
\newblock {\em arXiv preprint arXiv:2411.08302}, 2024.

\bibitem{xu2024dpo}
Shusheng Xu, Wei Fu, Jiaxuan Gao, Wenjie Ye, Weilin Liu, Zhiyu Mei, Guangju
  Wang, Chao Yu, and Yi~Wu.
\newblock Is dpo superior to ppo for llm alignment? a comprehensive study.
\newblock {\em arXiv preprint arXiv:2404.10719}, 2024.

\bibitem{sutton2018reinforcement}
Richard~S Sutton and Andrew~G Barto.
\newblock Reinforcement learning: An introduction. second, 2018.

\bibitem{wang2025analysis}
Weizhen Wang, Jianping He, and Xiaoming Duan.
\newblock Analysis of on-policy policy gradient methods under the distribution
  mismatch.
\newblock {\em arXiv preprint arXiv:2503.22244}, 2025.

\bibitem{gae}
John Schulman, Philipp Moritz, Sergey Levine, Michael Jordan, and Pieter
  Abbeel.
\newblock High-dimensional continuous control using generalized advantage
  estimation, 2018.

\bibitem{kazemnejad2024vineppo}
Amirhossein Kazemnejad, Milad Aghajohari, Eva Portelance, Alessandro Sordoni,
  Siva Reddy, Aaron Courville, and Nicolas~Le Roux.
\newblock Vineppo: Unlocking rl potential for llm reasoning through refined
  credit assignment.
\newblock {\em arXiv preprint arXiv:2410.01679}, 2024.

\bibitem{kwon2023efficient}
Woosuk Kwon, Zhuohan Li, Siyuan Zhuang, Ying Sheng, Lianmin Zheng, Cody~Hao Yu,
  Joseph~E. Gonzalez, Hao Zhang, and Ion Stoica.
\newblock Efficient memory management for large language model serving with
  pagedattention.
\newblock In {\em Proceedings of the ACM SIGOPS 29th Symposium on Operating
  Systems Principles}, 2023.

\bibitem{rafailov2023direct}
Rafael Rafailov, Archit Sharma, Eric Mitchell, Christopher~D Manning, Stefano
  Ermon, and Chelsea Finn.
\newblock Direct preference optimization: Your language model is secretly a
  reward model.
\newblock {\em Advances in Neural Information Processing Systems},
  36:53728--53741, 2023.

\bibitem{zhou2025sweet}
Yifei Zhou, Song Jiang, Yuandong Tian, Jason Weston, Sergey Levine, Sainbayar
  Sukhbaatar, and Xian Li.
\newblock Sweet-rl: Training multi-turn llm agents on collaborative reasoning
  tasks.
\newblock {\em arXiv preprint arXiv:2503.15478}, 2025.

\bibitem{shani2024multi}
Lior Shani, Aviv Rosenberg, Asaf Cassel, Oran Lang, Daniele Calandriello,
  Avital Zipori, Hila Noga, Orgad Keller, Bilal Piot, Idan Szpektor, et~al.
\newblock Multi-turn reinforcement learning from preference human feedback.
\newblock {\em arXiv preprint arXiv:2405.14655}, 2024.

\bibitem{abdulhai2023lmrl}
Marwa Abdulhai, Isadora White, Charlie Snell, Charles Sun, Joey Hong, Yuexiang
  Zhai, Kelvin Xu, and Sergey Levine.
\newblock Lmrl gym: Benchmarks for multi-turn reinforcement learning with
  language models.
\newblock {\em arXiv preprint arXiv:2311.18232}, 2023.

\bibitem{zhou2024archer}
Yifei Zhou, Andrea Zanette, Jiayi Pan, Sergey Levine, and Aviral Kumar.
\newblock Archer: Training language model agents via hierarchical multi-turn
  rl.
\newblock {\em arXiv preprint arXiv:2402.19446}, 2024.

\bibitem{wang2024offline}
Huaijie Wang, Shibo Hao, Hanze Dong, Shenao Zhang, Yilin Bao, Ziran Yang, and
  Yi~Wu.
\newblock Offline reinforcement learning for llm multi-step reasoning.
\newblock {\em arXiv preprint arXiv:2412.16145}, 2024.

\bibitem{wei2025plangenllms}
Hui Wei, Zihao Zhang, Shenghua He, Tian Xia, Shijia Pan, and Fei Liu.
\newblock Plangenllms: A modern survey of llm planning capabilities.
\newblock {\em arXiv preprint arXiv:2502.11221}, 2025.

\bibitem{bai2024mt}
Ge~Bai, Jie Liu, Xingyuan Bu, Yancheng He, Jiaheng Liu, Zhanhui Zhou, Zhuoran
  Lin, Wenbo Su, Tiezheng Ge, Bo~Zheng, et~al.
\newblock Mt-bench-101: A fine-grained benchmark for evaluating large language
  models in multi-turn dialogues.
\newblock {\em arXiv preprint arXiv:2402.14762}, 2024.

\bibitem{ai-patient}
Ruoyu Liu, Kui Xue, Xiaofan Zhang, and Shaoting Zhang.
\newblock Interactive evaluation for medical {LLM}s via task-oriented dialogue
  system.
\newblock In {\em Proceedings of the 31st International Conference on
  Computational Linguistics}, 2025.

\bibitem{lightman2023letsverifystepstep}
Hunter Lightman, Vineet Kosaraju, Yura Burda, Harri Edwards, Bowen Baker, Teddy
  Lee, Jan Leike, John Schulman, Ilya Sutskever, and Karl Cobbe.
\newblock Let's verify step by step, 2023.

\bibitem{jian2024}
Jian Hu.
\newblock Unraveling rlhf and its variants: Progress and practical engineering
  insights.
\newblock 2024.

\bibitem{uesato2022solving}
Jonathan Uesato, Nate Kushman, Ramana Kumar, Francis Song, Noah Siegel, Lisa
  Wang, Antonia Creswell, Geoffrey Irving, and Irina Higgins.
\newblock Solving math word problems with process-and outcome-based feedback.
\newblock {\em arXiv preprint arXiv:2211.14275}, 2022.

\bibitem{wu2023fine}
Zeqiu Wu, Yushi Hu, Weijia Shi, Nouha Dziri, Alane Suhr, Prithviraj
  Ammanabrolu, Noah~A Smith, Mari Ostendorf, and Hannaneh Hajishirzi.
\newblock Fine-grained human feedback gives better rewards for language model
  training.
\newblock {\em Advances in Neural Information Processing Systems},
  36:59008--59033, 2023.

\bibitem{cao2024beyond}
Meng Cao, Lei Shu, Lei Yu, Yun Zhu, Nevan Wichers, Yinxiao Liu, and Lei Meng.
\newblock Beyond sparse rewards: Enhancing reinforcement learning with language
  model critique in text generation.
\newblock {\em arXiv preprint arXiv:2401.07382}, 2024.

\bibitem{yoon2024tlcr}
Eunseop Yoon, Hee~Suk Yoon, SooHwan Eom, Gunsoo Han, Daniel~Wontae Nam, Daejin
  Jo, Kyoung-Woon On, Mark~A Hasegawa-Johnson, Sungwoong Kim, and Chang~D Yoo.
\newblock Tlcr: Token-level continuous reward for fine-grained reinforcement
  learning from human feedback.
\newblock {\em arXiv preprint arXiv:2407.16574}, 2024.

\bibitem{wang2023math}
Peiyi Wang, Lei Li, Zhihong Shao, RX~Xu, Damai Dai, Yifei Li, Deli Chen, Yu~Wu,
  and Zhifang Sui.
\newblock Math-shepherd: Verify and reinforce llms step-by-step without human
  annotations.
\newblock {\em arXiv preprint arXiv:2312.08935}, 2023.

\bibitem{lyu2025exploring}
Chengqi Lyu, Songyang Gao, Yuzhe Gu, Wenwei Zhang, Jianfei Gao, Kuikun Liu,
  Ziyi Wang, Shuaibin Li, Qian Zhao, Haian Huang, et~al.
\newblock Exploring the limit of outcome reward for learning mathematical
  reasoning.
\newblock {\em arXiv preprint arXiv:2502.06781}, 2025.

\bibitem{chiang2024chatbot}
Wei-Lin Chiang, Lianmin Zheng, Ying Sheng, Anastasios~Nikolas Angelopoulos,
  Tianle Li, Dacheng Li, Banghua Zhu, Hao Zhang, Michael Jordan, Joseph~E
  Gonzalez, et~al.
\newblock Chatbot arena: An open platform for evaluating llms by human
  preference.
\newblock In {\em Forty-first International Conference on Machine Learning},
  2024.

\bibitem{wei2024systematic}
Hui Wei, Shenghua He, Tian Xia, Fei Liu, Andy Wong, Jingyang Lin, and Mei Han.
\newblock Systematic evaluation of llm-as-a-judge in llm alignment tasks:
  Explainable metrics and diverse prompt templates.
\newblock {\em arXiv preprint arXiv:2408.13006}, 2024.

\bibitem{chai2024ma}
Yekun Chai, Haoran Sun, Huang Fang, Shuohuan Wang, Yu~Sun, and Hua Wu.
\newblock Ma-rlhf: Reinforcement learning from human feedback with macro
  actions.
\newblock {\em arXiv preprint arXiv:2410.02743}, 2024.

\end{thebibliography}
\bibliographystyle{unsrt}

\end{document}